\documentclass{article}

\usepackage{amsmath,amsfonts,bm,amssymb}

\def\eqref#1{equation~\ref{#1}}

\def\1{\bm{1}}

\def\vk{{\bm{k}}}

\def\vm{{\bm{m}}}

\def\vo{{\bm{o}}}
\def\vp{{\bm{p}}}
\def\vq{{\bm{q}}}

\def\vs{{\bm{s}}}

\def\vu{{\bm{u}}}
\def\vv{{\bm{v}}}
\def\vw{{\bm{w}}}
\def\vx{{\bm{x}}}

\def\vz{{\bm{z}}}

\DeclareMathAlphabet{\mathsfit}{\encodingdefault}{\sfdefault}{m}{sl}
\SetMathAlphabet{\mathsfit}{bold}{\encodingdefault}{\sfdefault}{bx}{n}

\def\sR{{\mathbb{R}}}

\DeclareMathOperator*{\argmax}{arg\,max}
\DeclareMathOperator*{\argmin}{arg\,min}

    \PassOptionsToPackage{numbers, compress}{natbib}

\usepackage[preprint]{neurips_2024}

\usepackage[utf8]{inputenc} 
\usepackage[T1]{fontenc}    
\usepackage[colorlinks]{hyperref}       
\usepackage{url}            
\usepackage{booktabs}       
\usepackage{amsfonts}       
\usepackage{nicefrac}       
\usepackage{microtype}      
\usepackage[dvipsnames]{xcolor}         
\usepackage{subcaption}
\usepackage{graphicx}
\usepackage{mathtools}
\usepackage{amsthm}
\usepackage{bbm}
\usepackage{mypackage}
\usepackage{soul}
\usepackage{algorithm}
\usepackage{algcompatible}

\usepackage{multicol}
\usepackage{multirow}
\usepackage{makecell}
\usepackage{wrapfig,lipsum,booktabs}

\theoremstyle{plain}
\newtheorem{theorem}{Theorem}[section]
\newtheorem{proposition}[theorem]{Proposition}

\theoremstyle{definition}

\theoremstyle{remark}

\title{Sparser is Faster and Less is More: Efficient Sparse Attention for Long-Range Transformers}

\author{%
Chao Lou$^{1}$ \quad Zixia Jia$^{2}$ \quad Zilong Zheng$^{2,*}$ \quad Kewei Tu$^{1,}$\thanks{Corresponding Author} \\
$^{1}$ ShanghaiTech University \\
$^{2}$ National Key Laboratory of General Artificial Intelligence, BIGAI \\
\texttt{\{louchao,tukw\}@shanghaitech.edu.cn}  \\
\texttt{\{jiazixia,zlzheng\}@bigai.ai} \\
}

\begin{document}

\maketitle

\begin{abstract}
    Accommodating long sequences efficiently in autoregressive Transformers, especially within an extended context window, poses significant challenges due to the quadratic computational complexity and substantial KV memory requirements inherent in self-attention mechanisms. In this work, we introduce $\sparsek$ Attention, a novel sparse attention mechanism designed to overcome these computational and memory obstacles while maintaining performance. Our approach integrates a scoring network and a differentiable top-k mask operator, $\sparsek$, to select a constant number of KV pairs for each query, thereby enabling gradient-based optimization. As a result, $\sparsek$ Attention offers linear time complexity and constant memory footprint during generation. Experimental results reveal that $\sparsek$ Attention outperforms previous sparse attention methods and provides significant speed improvements during both training and inference, particularly in language modeling and downstream tasks. Furthermore, our method can be seamlessly integrated into pre-trained Large Language Models (LLMs) with minimal fine-tuning, offering a practical solution for effectively managing long-range dependencies in diverse applications. Our code will be publicly available.
\end{abstract}

\section{Introduction}
\label{sec:intro}

Transformer models~\cite{Vaswani2017AttentionIA} have been considered as a \textit{de facto} backbone of modeling arbitrary sequences, pretraining foundation models~\cite{bommasani2021opportunities,devlin2018bert}, and more recently, constructing large language models~(LLMs) ~\cite{brown2020language,achiam2023gpt}. Despite the inspiring success of their wide applications on both Natural Language Processing (NLP) and Machine Learning (ML) downstream tasks, extending the context window size to long sequences with computation and  memory
 efficiently poses significant challenges~\cite{Ainslie2023CoLT5FL, Dao2022FlashAttentionFA, Dao2023FlashAttention2FA}, owing to the quadratic computation complexity and large amounts of key/value vectors associated with self-attention, especially on resource-constrained devices.

Many recent studies resort to developing learnable sparse and memory-efficient forms of attention to scale to large sequence lengths.
However, applying traditional learnable sparse attention methods to long-range Transformer decoders suffers from two major bottlenecks:
(i) Previous studies usually overlook the memory cost of \textbf{fully memorizing Key-Value (KV) pairs}. Clustering-based methods~\cite{Kitaev2020ReformerTE,Roy2020EfficientCS} allow queries to attend to different sets of KV pairs. In such methods, KV embeddings are required to be fully stored in memory to avoid repetitive computation, which leads to huge memory redundancy and inefficiency when it comes to long-range inference~\cite{Zhang2023H2OHO,Liu2023ScissorhandsET,Yu2023TRAMSTM}.
(ii) Previous learnable sparse attention often has \textbf{super-linear complexity}, especially during training. For example, clustering-based methods usually cost $O(n\log n)$ to maintain clusters. \citet{Ainslie2023CoLT5FL} incorporates a $\textsc{SoftTopK}$ operator~\cite{Lei2023ConditionalAP} to compute soft masks in Transformer encoders. Meanwhile, migrating $\textsc{SoftTopK}$ to Transformer decoders is less advantageous because solving $\textsc{SoftTopK}$ for variable-length context associated with different queries requires quadratic time in total.

To tackle the aforementioned barriers, we propose \textit{SparseK Attention}, an innovative technique that achieves both computational and memory efficiency for training and inference-time attention computing in Transformer decoders, as depicted in Figure~\ref{fig:main}. Within a self-attention module, our method incorporates (1) a scoring network evaluating the importance of each KV pair without accessing the queries that possibly attend to it, and (2) a novel differentiable top-$k$ mask operator $\sparsek$, which normalizes scores to a soft mask (or gates) in linear time. 
It is worth noting that our method draws inspiration from the concept of top-$k$ attention~\cite{Gupta2021MemoryefficientTV,Ainslie2023CoLT5FL}. Unfortunately, conventional top-$k$ attention is non-differentiable and therefore cannot be used to train the scoring network.
With thorough comparisons with prior sparse attention learning approaches, we highlight the main advantages of $\sparsek$ attention as follows.

\textbf{Incremental KV Selection.}

The $\sparsek$ operator ($\S$~\ref{sec:sparsek}) supports incremental evaluation and thus has a linear complexity in the decoder.
Besides, compared with $\textsc{SoftTopK}$ that performs iterative approximation as in CoLT5~\cite{Ainslie2023CoLT5FL}, our operator computes the exact operation results.

\textbf{Computational and Memory Efficiency.} $\sparsek$ reduces the quadratic training-time complexity of previous learnable sparse attention methods~\cite{Sukhbaatar2019AdaptiveAS,Gupta2021MemoryefficientTV,Anagnostidis2023DynamicCP,Mohtashami2023LandmarkAR} to linear time and achieves constant memory cost in inference.
This improvement of training-time complexity is achieved by the efficiency of KV selection and applying the same level of sparsity in training as in inference.
Additionally, the query-independence of our scoring network guarantees the irreversibility of masking out key-value pairs. This ensures memory efficiency at inference time, allowing for the safe removal of masked key-value pairs from memory immediately~($\S$~\ref{sec:kv_selection}).

\textbf{Extension with IO-awareness.} FlashAttention~\cite{Dao2022FlashAttentionFA} is a widely adopted optimization for accelerating LLMs with IO-awareness. However, the sparsity learned through our method presents a complex memory access pattern, hindering its direct application. To address this, we develop a Triton kernel that fuses the computation of attention and the selection of proper key-value pairs. Our implementation exhibits linear complexity and surpasses FlashAttention in performance when handling 4096 input tokens, of which 1024 key-value pairs are selected for each query. Additionally, we offer a kernel for the backward pass, which fuses the computation of the gradient of $\sparsek$ and others, resulting in increased speed and improved memory efficiency.

\begin{figure}[t]
    \centering
    \includegraphics[width=\textwidth]{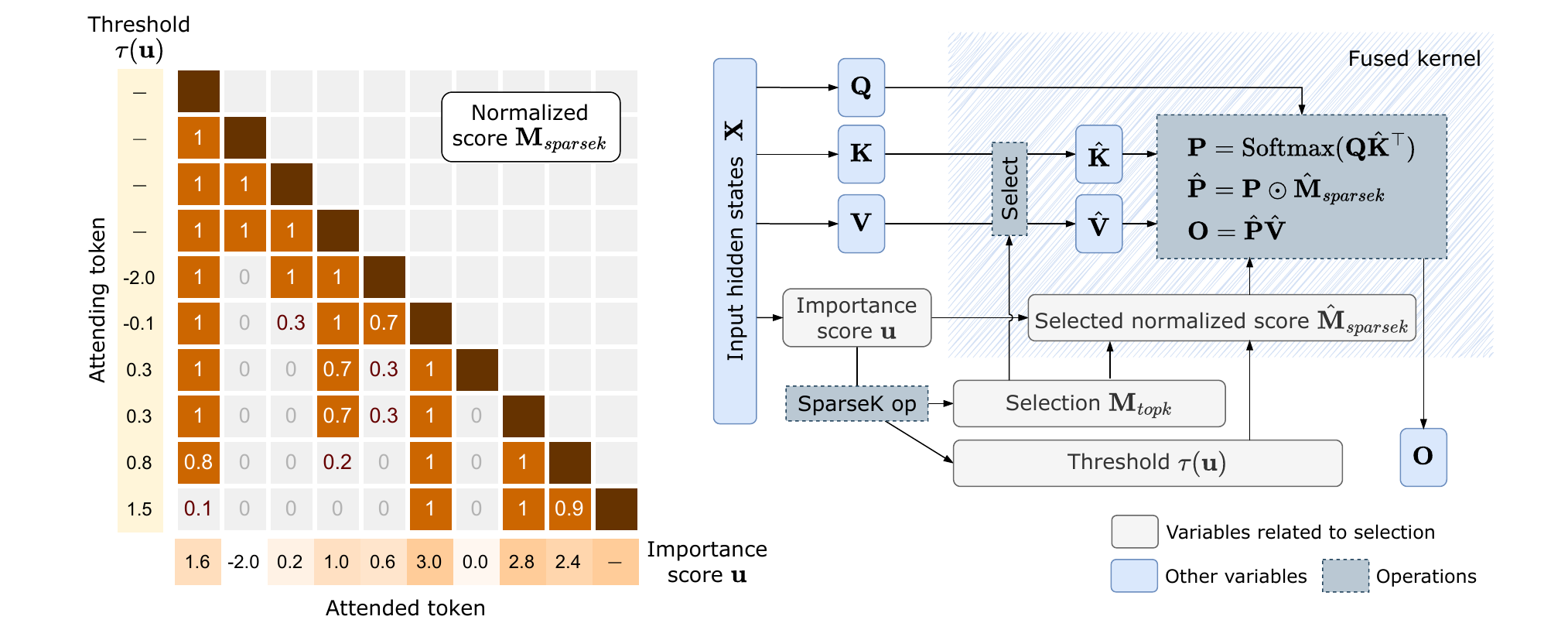}
    \caption{\textbf{Left: $\sparsek$ operation in the attention module.} KV pairs are scored by $\mathbf{u}$. $\sparsek$ computes a threshold for each query ($\tau(\mathbf{u})$) such that the sum of normalized scores is $k$, which is 3 in this example. We select top-$k$ KV pairs (orange cells) to perform attention. \textbf{Right: the $\sparsek$ attention module.} We fuse selection and attention in one kernel for efficiency. }
    \label{fig:main}
\end{figure}

We verify the advantages of $\sparsek$ attention by replacing full attention in various models (such as GPT2~\cite{radford2019language} and Pythia~\cite{biderman2023pythia}) with it and other efficient attention methods. We consider a wide range of settings, including training from scratch and fine-tuning pretrained models. Experiments on language modeling and downstream tasks demonstrate that, when matching the context size, our method outperforms other efficient attention methods consistently while providing promising speed-up at training compared to full attention.

\section{Related Work}

\paragraph{Long-range Transformers}
\label{sec:eff_trfms}

Self-attention is a cornerstone of Transformer success, but its quadratic complexity concerning input length poses challenges for tasks requiring long context. Numerous efficient approaches have emerged, spanning state-space models~\cite{Gu2021EfficientlyML,Smith2022SimplifiedSS}, recurrent neural networks~\cite{Martin2017ParallelizingLR,Peng2023RWKVRR,Orvieto2023ResurrectingRN}, linear attention~\cite{qin2024transnormerllm,Katharopoulos2020TransformersAR} and low-rank approximations of self-attention~\cite{Wang2020LinformerSW,Choromanski2020RethinkingAW,Peng2021RandomFA}, which replace the self-attention with novel linear blocks for long-context modeling. Nonetheless, these approaches historically underperformed compared to modern Transformer models~\cite{Touvron2023LLaMAOA} in language modeling tasks until recent efforts~\cite{Gu2023MambaLS,Yang2023GatedLA}. Besides, a few studies combine the Transformer with block-wise recurrence~\cite{Dai2019TransformerXLAL,Hua2022TransformerQI,Hutchins2022BlockRecurrentT,Chevalier2023AdaptingLM} or key-value compression~\cite{Ren2021CombinerFA,Rae2019CompressiveTF,Dai2020FunnelTransformerFO}.
In contrast, our approach falls under sparse attention, reducing complexity by pruning the attention matrix. This approach is motivated by observations that the attention matrix in dense models naturally becomes sparse, and the performance of language models remains robust under reasonably sparse conditions~\cite{Clark2019WhatDB,Ge2023ModelTY,Liu2023ScissorhandsET}.

\paragraph{Sparse attention}

Some sparse attention utilized fixed patterns to restrict the number of tokens involved, such as sliding windows~\cite{Qiu2019BlockwiseSF,Parmar2018ImageT}, dilated sliding windows~\cite{Beltagy2020LongformerTL,Ding2023LongNetST}, combination of patterns~\cite{Ho2019AxialAI,Child2019GeneratingLS}, or domain-specific patterns~\cite{Guo2023LongCoderAL}.
Recent studies have aimed at achieving constant memory costs during inference through predefined heuristic cache eviction policies~\cite{Zhang2023H2OHO,Liu2023ScissorhandsET,Ge2023ModelTY}. However, these static methods often prove suboptimal in various scenarios~\cite{Sun2021DoLL,Anagnostidis2023DynamicCP}. Alternatively, sparse patterns can be learned in a data-driven manner. For example, Reformer~\cite{Kitaev2020ReformerTE} employs locality-sensitive hashing for token clustering and do attention within a cluster, while Routing Transformers~\cite{Roy2020EfficientCS}, Cluster-Former~\cite{Wang2020ClusterFormerCS} and Clustered Attention~\cite{Vyas2020FastTW} use K-Means clustering on tokens. Besides, Sparse Sinkhorn Attention~\cite{Tay2020SparseSA} establishes sparsity by sorting blocks of inputs. Despite achieving sub-quadratic complexity, these methods still remain above linear complexity and face challenges when handling extremely long sequences or failing to offer constant memory cost during inference. A recent approach by \citet{Anagnostidis2023DynamicCP} introduces a learnable, irreversible key-value pair pruning for inference-time memory efficiency with the concept of relaxing pruning actions to accumulated gating. However, this method still suffers from quadratic complexity during training, hindering its ability to expedite the training process. In this paper, we present a novel, efficient sparse attention mechanism with learnable patterns, addressing all the aforementioned challenges.

\section{SparseK Attention}

\subsection{Background}
\paragraph{Self-Attention}

Given a sequence of vectors $\bX\in\sR^{n\times d}$ where $n$ is the sequence length and $d$ is the hidden dimension, an attention head first projects $\bX$ into query, key and value vectors with $\bW_{Q},\bW_{K},\bW_{V}\in \sR^{d\times p}$ where $p = \frac{d}{h}$ and $h$ is the number of attention heads:
\begin{align}
    \bQ=\bX \bW_{Q} &  & \bK=\bX \bW_{K} &  & \bV=\bX \bW_{V},
\end{align}
In the decoder-only architecture~\cite{Vaswani2017AttentionIA}, a causal attention mask $\bM$ guarantees each query $q_i$ only attends to positions $\le i$. Consequently, the output $\bO$ of single-head dot-product attention is defined as
\begin{align}
    \bS = \bQ\bK^\top \qquad  \bP = \textsc{SoftMax}(\bS + \bM) \qquad \bO = \bP\bV
\end{align}
The multi-head self-attention concatenates the outputs of multiple heads (indexed by subscripts) and applies a linear projection with $\bW_O\in \sR^{d\times d}$:
\begin{align}
    \text{MHA}(\bX) = \text{Concatenate}(\bO_1, \bO_2, \dots, \bO_h)\bW_O
\end{align}
The quadratic complexity of self-attention is contributed by the quadratically sized attention weight $\bS$.  Inspired by~\citet{Ainslie2023CoLT5FL}, we propose to select a constant number of key-value pairs for each query in an irreversible way (defined formally in the following subsections \ref{sec:kv_selection} and \ref{sec:sparsek}), leading to linear training complexity and a constant inference-time memory cost. For simplicity, here we omit the RoPE position embedding~\cite{Su2021RoFormerET} and focus on single-head attention to illustrate our methodology. The multi-head case is briefly discussed in Appendix~\ref{sec:multi_head_selection}.

\paragraph{SparseMax operator}
There are many popular technical choices that relax $\textsc{ArgMax}$ operation, such as $\textsc{SoftMax}$ and $\textsc{SparseMax}$~\citep{Martins2016FromST}.
Especially, $\textsc{SparseMax}$ uses the Euclidean projection onto the probabilistic simplex and tends to yield sparse solutions:
\begin{align}
    \textsc{SparseMax}(\vz)
     & \coloneqq \argmin_{\vp\in \triangle^{m-1}} || \vp-\vz ||^2,\label{eq:sparsemax}
\end{align}
where $\triangle^{m-1}=\{\vp\in \sR^m | \bm{1}^\top \vp = 1, \vp\ge0\}$.
Building on this, we introduce $\sparsek$, an extension of $\textsc{SparseMax}$ for the case where $k=\bm{1}^\top \vp\ge1$.

\subsection{Learnable Key-Value Pair Selection}\label{sec:kv_selection}

\paragraph{Key-value pair selection} We use $\Delta\in \{0,1\}^{k\times m}$ to represent the selection of $k$ key-value pairs out of $m$ entries, where $\Delta(i,j)=1$ indicates that the $j$-th key-value pair is the $i$-th selected entry ( \textit{i.e.}, the $j$-th key-value pair is positioned in the $i$-th slot after sorting), and $\Delta(i,j)=0$ otherwise. It is noteworthy that each column vector and row vector exhibit one-hot characteristics. We use subscripts to distinguish selection corresponding to different queries, i.e., $\Delta_t$ for $\vq_t$. The causality of Transformer decoders puts a natural constraint: $\Delta_t(i,j)=0$ if $j>t$. Then, the self-attention with query $\vq_i$ and its selected contexts are defined as
\begin{align}
    \begin{aligned}
     & \vq_i=\bW_Q\vx_i \qquad \hat{\bK}_i = \Delta_{i} \bK \qquad  \hat{\bV}_i = \Delta_{i} \bV                                                             &  \\
     & \hat{\vs}_i=\vq_i^\top\hat{\bK}_i  \quad  \vo_i = \hat{\vp}_i^\top\hat{\bV}_i  \quad \hat{\vp}_i = \textsc{SoftMax}(\hat{\vs}_i), &
    \end{aligned}\label{eq:naive_imp}
\end{align}

\paragraph{Irreversibile selection} In a series of decoder steps, the selection of key-value pairs for queries $\{\vq_1, \vq_2, \dots, \vq_n\}$ is performed over incremental candidate sets (e.g., $\{\vk_1\}, \{\vk_1, \vk_2\}, \dots, \{\vk_i\}_{i=1}^n$). Irreversible selection is a strategy that, at step $t$,  $\vk_i$ ($i<t$) can be selected only if it has been selected in all preceding steps from $i$ to $t-1$, $\forall t' \in \{i, i+1, \dots, t-1\}, \exists j: \Delta_{t'}(j, i) = 1.$
Irreversible selection is memory-efficient because it guarantees the feasibility of pruning a key-value pair once it is not selected at a step. Many previous studies, such as clustering-based methods~\cite{Roy2020EfficientCS,Kitaev2020ReformerTE}, put no constraint on interactions between queries and key-value pairs and thus are not irreversible, resulting in a memory cost equivalent to that of pure self-attention.

We employ a simple selection strategy in this paper: we score each key-value pair and choose the key-value pairs with the top-$k$ scores. This strategy is irreversible because of the monotonic increase of the $k$-th largest value in incremental sets. Note that we do not use attention scores as the selection criteria~\cite{Gupta2021MemoryefficientTV} to avoid the quadratic cost of computing $\bQ\bK^\top$.

\paragraph{Learnable selection}
Formally, we use a scoring network, a linear projection with parameter $\vw_{score}\in \sR^{d}$ augmented with a position slope $d(\vu) \in \sR^{n}$, to get the importance scores $\vu$ of the key-value pairs computed from the input vectors $\bX$:
\begin{align}
    \vu = \bX\vw_{score} + d(\vu), \qquad\text{where}\quad d(\vu)_i = i\epsilon.\label{eq:compute_u}
\end{align}
for some small positive scalar $\epsilon$.   The introduction of distance slope is to avoid numerical explosion of $\bX\vw_{score}$: if without it, the scoring network would be compelled to predict increasingly larger scores to retain new tokens and discard old ones. This hurts training stabilitiy and generalization ability to larger context length. 

Then, the selection matrix $\Delta_t$ for query $q_t$ can be defined to select top-$k$ entries,
\begin{align}
    \begin{aligned}
     & \vm_{topk} = \textsc{TopK}(\vu_{1:i}, k)         \\
     & \Delta_t^{hard} = \text{MaskSelect}(\text{Diag}(\vm_{topk}), \vm_{topk}),
    \end{aligned} \label{eq:selection}
\end{align}
where $\vm_{topk}$ is an indicator vector, i.e., $\vm_{topk}(j)=1$ if $u_j$ ranks within the top-$k$ of $\vu_{1:i}=\{u_1, u_2, \dots, u_i\}$ and $\vm_{topk}(j)=0$ otherwise, and $\text{Diag}(\vm)$ is a matrix with $\vm$ in the main diagonal. The function $\text{MaskSelect}(\bX, \vm)$ selects rows of $\bX$ according to the mask $\vm$.

$\textsc{TopK}$ is not differentiable with respect to its input, preventing us from updating $\vw_{score}$ with gradient-based methods. Therefore, we propose to use $\sparsek$, a differentiable relaxation of $\textsc{TopK}$, in Equation (\ref{eq:selection}),
\begin{align}
    \begin{aligned}
     & \vm_{sparsek} = \sparsek(\vu_{1:i},k) \\
     & \Delta_t^{soft} = \text{MaskSelect}(\text{Diag}(\vm_{sparsek}), \vm_{topk}),
    \end{aligned} \label{eq:sparse_k_selection}
\end{align}

\subsection{The Differentiable SparseK Operator}\label{sec:sparsek}

\paragraph{Definition}

We relax the constraint in Equation (\ref{eq:sparsemax}) from a probablistic simplex to a k-sum constraint $\mathbb{C} = \{\vp | \bm{0} \le \vp \le \1,  \1^\top\vp=k\}$ and define $\sparsek$ as follows,
\begin{align}
    \sparsek(\vz, k) & \coloneqq \argmin_{\vp\in \mathbb{C}}|| \vp-\vz ||^2\label{eq:sparsek_def} \\
                     & = \argmax_{\vp\in \mathbb{C}} \vp^\top \vz + H^G(\vp),
\end{align}
where $H^G(\vp) = \frac{1}{2}\sum_j p_j(1-p_j)$ is the generalized Gini entropy for $\vp \in \mathbb{C} $ instead of a distribution.

$\sparsek$ is related to SoftTopK as explained below. The generalized Gini entropy is a special case of the generalized $\alpha$-Tsallis entropy $H^T(\vp)$ for $\alpha=2$~\cite{tsallis_possible_1988}, where $H^T(\vp) = \frac{1}{\alpha(\alpha - 1)} \sum_j (p_j - p_j^\alpha)$ for $\vp \in \mathbb{C}$. Then we can define the generalized $\alpha\textsc{-EntTopK}$ operator as
\begin{align}
    \alpha\textsc{-EntTopK} = \argmax_{\vp\in \mathbb{C}} \vp^\top \vz + H^T(\vp).
\end{align}
The entropic index $\alpha$ controls the smoothness of the solution. Taking the limit of $\alpha\rightarrow \infty$, we get the $\textsc{TopK}$ operator. Taking the limit of $\alpha\rightarrow 1$, the generalized $\alpha$-Tsallis entropy becomes the generalized Gibbs-Boltzmann-Shannon entropy and we obtain the $\textsc{SoftTopK}$ operator~\cite{Lei2023ConditionalAP}.

\paragraph{Solution} Similar to $\textsc{SparseMax}$, $\textsc{SparseK}$ is a soft-thresholding operation. Its solution is expressed as follows:
\begin{align}
    \vp^* &= \max(\min(\vz - \tau(\vz), \bm{1}), \bm{0}),\label{eq:sparsek_sol_main} \\
    \tau(\vz) & = \frac{\sum_{u^* < j \le w^*} z_{(j)} + u^* - k}{ w^* - u^* },\label{eq:run_sol}
\end{align}
where $\tau(\vz): \sR^n \rightarrow \sR$ is the threshold function that satisfies $\sum \vp^* = k$, $z_{(1)} > z_{(2)} > \dots > z_{(m)}$ is the sorted coordinates of $\vz$, $u^*$ is the number of entries with value 1 in $\vp^*$, and $w^*$ is the number of entries with nonzero values. 
Algorithm~\ref{alg:one_step} illustrates an $O(m\log m)$ algorithm for evaluating $\sparsek$. With the pre-computed cumulative sum of $\vz$, line 5 can be evaluated in $O(1)$, and thus the overall complexity is primarily due to sorting. We give the proof in Appendix~\ref{sec:proof}. 

\begin{figure}[ht]
    \begin{minipage}[t]{\textwidth}
        \begin{multicols}{2}
            \begin{algorithm}[H]
                \caption{Evaluate $\sparsek(\vz, k)$ }
                \label{alg:one_step}
                \begin{algorithmic}[1]
                    \STATE {\bfseries Input:}  $\vz$
                    \STATE Sort $\vz$ as $z_{(1)} \ge \ldots \ge z_{(m)}$
                    \STATE Pre-compute the cumulative sum of $\vz$
                    \FOR{$(u, w)$ in the descending order}
                    \STATE $\tau \leftarrow$ as in Equation (\ref{eq:run_sol})\label{alg:line_eval_sol}
                    \IF{$z_{(w)} > \tau$ and $z_{(u)} \ge \tau + 1 $ }
                    \STATE \textbf{break}
                    \ENDIF
                    \ENDFOR
                    \STATE $\vp \leftarrow \max(\min(\vz - \tau, \bm{1}), \bm{0})$
                    \STATE {\bfseries Output:}  $\vp$
                \end{algorithmic}
            \end{algorithm}
            \vspace{-0.5cm}
            \setcounter{algorithm}{2}
            \begin{algorithm}[H]
                \caption{Train with chunk-wise recurrency}
                \label{alg:recurrency}
                \begin{algorithmic}[1]
                    \STATE {\bfseries Input:}  $\bX = [\bX_1, \bX_2, \dots, \bX_l]$
                    \STATE Initialize a KV cache with scores
                    \FORALL{$\bX_i$ in $\bX$}
                    \STATE $\bO_i \leftarrow$ attention with $\bX_i$ and the KV cache
                    \STATE Add new KV pairs and scores to the cache
                    \STATE Prune the KV cache to size $k$ by scores
                    \STATE Stop gradients of the KV cache
                    \ENDFOR
                    \STATE {\bfseries Output:}  $[\bO_1, \bO_2, \dots, \bO_l]$
                \end{algorithmic}
            \end{algorithm}
            \columnbreak
            \setcounter{algorithm}{1}
            \begin{algorithm}[H]
                \caption{Evaluate $\sparsek(\vz_{1:t}, k)$ at step $t$ from the result of step $t-1$}
                \label{alg:multi_step}
                \begin{algorithmic}[1]
                    \STATE {\bfseries Input:}  $z_t$, min-heaps $\mathcal{F} = \{z_i | z_i \ge \tau(z_{1:t-1}) +1\}$ and $\mathcal{S} = \{z_i | z_i > \tau(z_{1:t-1})\}$, $\tau_{t-1}$ from step $t-1$

                    \color{NavyBlue}
                    \STATEx \st{Sort $\vz$ as $z_{(1)} \ge \ldots \ge z_{(m)}$}
                    \STATEx \st{Pre-compute the cumulative sum of $\vz$}
                    \IF{$z_t \ge \tau(z_{1:t-1}) +1$}\label{lst:add_start}
                    \STATE Insert $z_t$ into $\mathcal{F}$
                    \ENDIF
                    \IF{$z_t > \tau(z_{1:t-1})$}
                    \STATE Insert $z_t$ into $\mathcal{S}$
                    \ENDIF\label{lst:add_end}
                    \color{black}
                    \FOR{$(u, w)$ in the descending order {\color{NavyBlue}from ($|\mathcal{S}|, |\mathcal{F}|$)}}
                    \STATE $\tau \leftarrow$ as in Equation (\ref{eq:run_sol})\label{alg:line_eval_sol_in_multi}
                    {\color{NavyBlue} \STATE Prune $\mathcal{S}$ and $\mathcal{F}$ with the new $\tau$}
                    \IF{$z_{(w)} > \tau$  and $z_{(u)} \ge \tau + 1 $ }
                    \STATE \textbf{break}
                    \ENDIF
                    \ENDFOR
                    \STATE $\vp \leftarrow \max(\min(\vz - \tau, \bm{1}), \bm{0})$
                    \STATE {\bfseries Output:}  $\vp, \mathcal{S}, \mathcal{F}, \tau$
                \end{algorithmic}
            \end{algorithm}
        \end{multicols}
    \end{minipage}
\end{figure}

Algorithm~\ref{alg:one_step} can be extended to evaluate $\sparsek$ on incremental sets. The idea is that we can compute step $t$ based on the results from step $t-1$ instead of starting from scratch.
Algorithm~\ref{alg:multi_step} illustrates the algorithm, where highlighted lines are the main difference from Algorithm~\ref{alg:one_step}. We introduce two min-heaps (and maintain the sum of elements for each heap) for tracking the search progress of $(u,w)$ and achieving $O(1)$ evaluation of line 9 in Algorithm~\ref{alg:multi_step}. Note that each insertion into a min-heap costs logarithmic time in the heap size and each $z_t$ introduces at most two more possible $(u, w)$ pairs (lines 2-6 in Algorithm~\ref{alg:multi_step}).
Therefore, executing Algorithm~\ref{alg:multi_step} over $m$ incremental sets (\textit{i.e.}, $m$ steps) costs $O(m\log m)$ in total.

As \citet{peters-etal-2019-sparse} have noted, the solution $\vp^*$ tends to contain only a few nonzeros, leading to small $u^*$ and $w^*$. Therefore, in practice, we can use partial sort on the $k'=O(k)$ largest values instead of full sort in Algorithm~\ref{alg:one_step}, thereby achieving a complexity of $O(m\log k)$. With respect to Algorithm~\ref{alg:multi_step}, this change is equivalent to restricting the size of the min-heap $\mathcal{S}$ to an upper bound for achieving the same reduction in complexity.

\subsection{Extensions}

\paragraph{Training with fixed-size truncation-free cache} Our selection method enables training on extremely long documents that need to be segmented into smaller chunks for recurrent processing. Algorithm~\ref{alg:recurrency} illustrates the process. Without introducing any additional truncation strategies or parameters, the algorithm maintains a fixed-size cache benefit by recurrent calculations and produces exactly the same results as calculating without chunking, which is guaranteed by the irreversibility of our selection method. To minimize the memory footprint, we stop the gradients of the cache, thereby pruning the computation graph, as in Transformer-XL~\cite{Dai2019TransformerXLAL}. With this algorithm, we can extend the training context length to hundreds of thousands of tokens.

\paragraph{Combine with other efficient attention mechanism}
\label{sec:combine_sparsek_and_local}

Our $\sparsek$ attention can be combined with other sparse attention as long as they have irreversible selection patterns. In this work, we integrate $\sparsek$ attention with sliding window (SW) attention by default, motivated by the well-known experience that sliding windows are simple yet incredibly strong for language modeling~\cite{rae-razavi-2020-transformers,Jiang2023Mistral7}.  Specifically, given a sliding window size $w$, we replace $\hat{\bK}_i, \hat{\bV}_i$ in (\ref{eq:naive_imp}) with
\begin{align}
    \hat{\bK}_i = \left[\begin{matrix}\Delta_{i-w} \bK \\ \bK_{i-w + 1:i}\end{matrix}\right] \quad
    \hat{\bV}_i = \left[\begin{matrix}\Delta_{i-w} \bV \\ \bV_{i-w + 1:i}\end{matrix}\right],
\end{align}
This combination does not introduce any overhead thanks to our fused Triton kernel. In this combination, $\sparsek$ attention attention aims at efficiently global (long-range) dependencies modeling, while SW attention is used for modeling local dependencies. 

Besides, $\sparsek$ attention can also be combined with linear attention methods, which hypothesize the existence of low-rank structures in attention scores rather than sparsity. From a theoretical perspective, \citet{Chen2021ScatterbrainUS} reveal that linear attention and sparse attention capture different attention patterns, and their combination provides a closer approximation to full attention. In this work, we extend their results to $\sparsek$ attention and recent attention optimizations~\cite{Dao2023FlashAttention2FA}. For technical details, please refer to Appendix~\ref{sec:combine_with_linear_attn}.

\paragraph{Straight-through estimator} From $\textsc{TopK}$ to $\sparsek$, we employ relaxation techniques to facilitate gradient-based training. Alternatively, the straight-through estimator (ST)~\cite{Bengio2013EstimatingOP} can be utilized, i.e., $\Delta^{st} = \Delta^{soft} - \text{stop\_grad}(\Delta^{soft}) + \Delta^{hard}$, allowing the model to perform true selection. By utilizing the ST method, the model achieves slightly improved efficiency since it bypasses the multiplication of selection scores during the forward pass. Our experimental results indicate that employing ST results in negligible performance degradation. Consequently, this technique shows promise in balancing performance and computational efficiency.

\subsection{Techniques for Faster and More Stable Training} 
\label{sec:faster_and_more_stable_training}

We introduce three beneficial modeling tricks discovered in our experiments. We also develop an optimized implementation based on FlashAttention-2\footnote{\url{https://github.com/openai/triton/blob/main/python/tutorials/06-fused-attention.py}} for obtaining practically efficient sparse attention. Please refer to Appendix~\ref{seq:eff_impl} for details.

\paragraph{Score normalization} We add timestep normalization~\cite{ma2024megalodon} on the time dimension: $\mathbf{u}^\prime = \text{TimestepNorm}(\mathbf{u})$, which computes cumulative mean and variance for each timestep. Note that the gradient of $\sparsek$ operation is sparse due to thresholding. We hypothesize that the additional normalization enables gradients in every position, thus resulting in better training. Moreover, timestep normalization is necessary to avoid numerical explosion when combining $\sparsek$ attention with linear attention. More discussion is in Appendix~\ref{sec:combine_with_linear_attn}.

\paragraph{Hard selection for keys and soft selection for values} Recall that the $\sparsek$ attention is given by $\textsc{SoftMax}(\bQ(\Delta^{soft}\bK)^\top)(\Delta^{soft}\bV)$. We have found it generally beneficial to use $\textsc{SoftMax}(\bQ(\Delta^{hard}\bK)^\top)(\Delta^{soft}\bV)$ instead, especially when fine-tuning pretrained models. This preference arises because the gradient of $\Delta^{soft}$ via $\Delta^{soft}\bK$ can be problematic given that, in pretrained models, some entries of $\bQ\bK^\top$ can be very close to the limitations of the \texttt{bfloat16} data type, whereas $\bV$ and the output of $\textsc{SoftMax}$ are generally stable.

\paragraph{Initialization to mimic attention score} When integrating our $\sparsek$ attention into pre-trained models, we hypothesize that an effective initialization for $\vw_{score}$ should ensure that the importance scores $\mathbf{u}$ correspond with cumulative attention scores, preventing key-value pairs that receive heavy attention from being pruned. Inspired by the selection metric proposed in \citet{Yu2023TRAMSTM}, whose ranking has been demonstrated to correlate strongly with the rankings of attention weights, we use $\bW_Q,\bW_K$ within the pretrained model for initialization: $\vw_{score} = \vw^\prime / \|\vw^\prime\|$ where $  \vw^\prime=\bW_Q\bW_K^\top\bm{1}$.

\section{Experiments}\label{sec:exp}

To evaluate the efficiency, scalability, and compatibility of $\sparsek$ attention, we tested it across various model architectures and scales. Our experiments focus on language modeling tasks using the OpenWebText corpus~\cite{Gokaslan2019OpenWeb} and the SlimPajama corpus~\cite{cerebras2023slimpajama}, and downstream tasks in LongBench~\cite{Bai2023LongBenchAB}. Empirically, the $\sparsek$ attention mechanism outperforms all previous efficient attention methods. More results, such as speed benchmark, ablation study and visualization, can be found in Appendix~\ref{sec:additional_results}.

\subsection{Language Modeling from Scratch}

\begin{wraptable}{R}{0.6\textwidth}
    \caption{Perplexity on the OpenWebText held-out set.}
    \label{tab:ppl_from_scratch}
    \centering
    \begin{tabular}{l|ccc} \toprule
       \multirow{ 2}{*}{\textbf{Model}} & \multicolumn{3}{c}{\textbf{Training Context Length}} \\
       & 1024 & 4096 & 8192 \\\midrule
       Full attention & 23.13 & 21.64 & 21.86 \\
       SW & 23.90 & 23.99 & 23.10 \\
       Linear + SW  &  23.27 & 22.97 & 23.23 \\
       Fixed & 23.26 & 22.47 & 22.86 \\
       Random  & 30.77 & 34.49 & 49.76\\
       Hash  & 26.53 & 27.42 & 27.58 \\\midrule
       GLA & 23.29 & 22.36 & 24.15 \\
       RetNet & 24.55 & 24.50 & 26.75 \\\midrule
       $\sparsek$ + SW & 22.85 & 21.98 & 21.84 \\
        $\sparsek$ + Linear + SW  & 22.46 & 21.55 & 21.32 \\\bottomrule
    \end{tabular}
\end{wraptable}

We adopt the GPT-2 small architecture, comprising 124 million parameters~\cite{radford2019language}. Notably, we substitute the original absolute position embeddings with rotary position embeddings as proposed in \citet{Su2021RoFormerET}. We alter the standard full attention in this architecture with our $\sparsek$ attention and several other efficient attention methods for comparison. For sparse attention, we include sliding window attention (SW)~\cite{Qiu2019BlockwiseSF,Parmar2018ImageT}, fixed sparse attention (Fixed)~\cite{Child2019GeneratingLS}, randomized sparse attention (Random)~\cite{Pagliardini2023FasterCA} and hash attention (Hash)~\cite{Kitaev2020ReformerTE,Pagliardini2023FasterCA}.  We adjust sparsity configurations to restrict the context window size to about 256 when the context length is 1024 or 4096 and about 512 when 8192. For linear attention, we utilize the kernelization proposed by \citet{Katharopoulos2020TransformersAR}. Additionally, we compare our methods to recent linear attention works that employ their own architectures rather than the GPT-2 architecture: GLA~\cite{Yang2023GatedLA} and RetNet~\cite{Sun2023RetentiveNA}. As the smallest GLA and RetNet is 340M, We modify their hyperparameters to align with our setting. Detailed hyperparameters and results of other configurations can be found in Appendix~\ref{sec:more_results_from_scratch}. 

We trained all models on the OpenWebText\footnote{\url{https://huggingface.co/datasets/Skylion007/openwebtext}} corpus for 10,000 steps, varying the context length. The results are presented in Table~\ref{tab:ppl_from_scratch}. Our $\sparsek$+SW method consistently outperforms all previously established efficient attention methods. Particularly, $\sparsek$+SW offers superior performance and has lower time complexity compared to previous learnable sparse attention methods, such as hash attention. Furthermore, linear attention methods, such as Linear+SW, GLA, and RetNet, demonstrate limitations, particularly in modeling long contexts. However, when combining linear attention with $\sparsek$ attention, we observed additional performance gains over $\sparsek$+SW, even surpassing full attention. This suggests the potential of exploring a mixture of different attention methods for more efficient modeling.

\subsection{Fine-tuning Existing Models}

\begin{figure}[ht]
    \centering
    \includegraphics[width=0.98\textwidth]{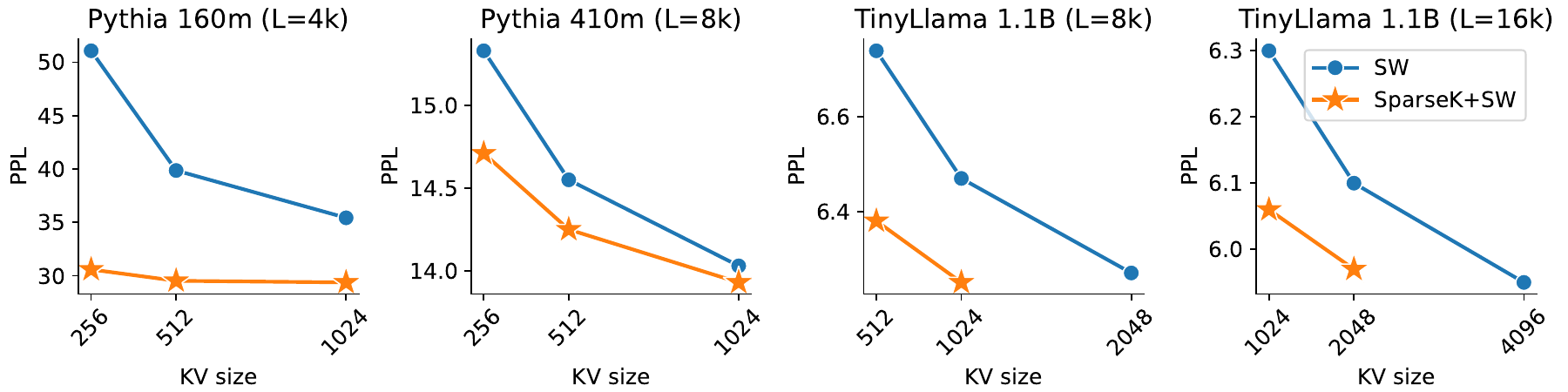}
    \caption{Perplexity on the held-out set of fine-tuned models. L denotes the training context length.}
    \label{fig:ft_ppl}
\end{figure}

We replace the standard full attention in Pythia 160M, Pythia 410M~\cite{biderman2023pythia} and TinyLlama 1.1B ~\cite{zhang2024tinyllama} with our $\sparsek$ attention and sliding window attention. The models are then fine-tuned over a few steps to ensure compatibility with the modified attention modules. Here we only consider sliding window attention because other efficient attention methods often require additional changes of the model architecture and sliding window attention is reported to be efficient in \citet{Chen2023LongLoRAEF}. In fine-tuning, the NTK-aware interpolation~\cite{blocntkaware} is adopted to extend the limit of pretrained positional encodings. For the Pythia models, we utilize a 1\% sampled subset of the SlimPajama dataset\footnote{\url{https://huggingface.co/datasets/DKYoon/SlimPajama-6B}}~\cite{cerebras2023slimpajama} to perform fine-tuning on moderate-length settings (i.e., 4k and 8k). In contrast, we use an upsampled dataset comprising long documents~\cite{Fu2024DataEF} to fine-tune the TinyLlama models on long-length settings (i.e., 8k and 16k). Training hyperparameters are listed in Appendix~\ref{sec:ft_hp}.

In Figure~\ref{fig:ft_ppl}, we report the perplexity on the held-out set across various levels of sparsity and training context lengths. Extending the training context length and increasing the context size generally benefit all types of attention mechanisms. When matching the KV size, our $\sparsek$+SW attention consistently outperforms sliding window attention. For the TinyLlama models, $\sparsek$+SW attention achieves comparable perplexity using only half the KV size required by sliding window attention. These results underscore the advantages of a more adaptable context as implemented in $\sparsek$+SW. We further evaluate TinyLlama 1.1B, fine-tuned with an 8k context window, across additional tasks as presented in the following sections.

\subsection{Retrieval-based Evaluation and Length Extrapolation}

\begin{figure}[ht]
    \centering
     \begin{subfigure}[t]{0.49\textwidth}
        \includegraphics[width=\textwidth]{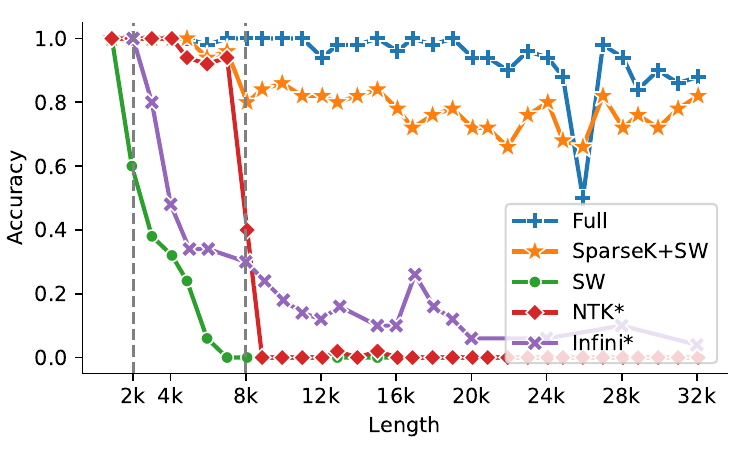}
        \caption{Passkey retrieval accuracy. }
        \label{fig:passkey}
    \end{subfigure}
    \hfill
    \begin{subfigure}[t]{0.49\textwidth}
        \includegraphics[width=\textwidth]{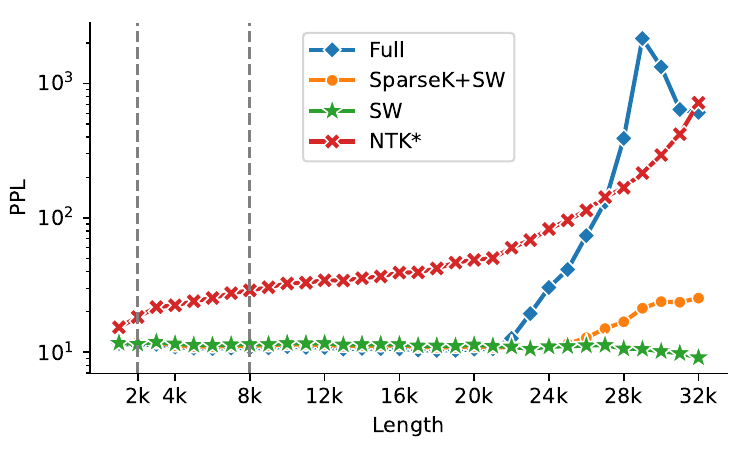}
        \caption{Perplexity of tokens in different position buckets on the PG19 test set. }
        \label{fig:pg19}
    \end{subfigure}
    \caption{Length extrapolation results. * denotes that the method is training-free.  2,048 is the context length of the original model. 8,192 is the context length in fune-tuning.}
\end{figure}

A common concern with sparse attention is its potential to neglect informative history. To investigate this, we evaluated our fine-tuned models on the passkey retrieval task~\cite{Mohtashami2023LandmarkAR}, along with two baseline methods that require no training: dynamic NTK~\cite{blocntkaware,dynamicntk} and LM-Infinite~\cite{han2023lminfinite}. The results are presented in Figure~\ref{fig:passkey}. It is evident that the sliding window approach fails even within the trained context length. Furthermore, among the training-free methods, NTK utilizes full attention and extends the context length by a factor of four, whereas the memory-efficient method LM-Infinite fails in extrapolation. In contrast, $\sparsek$+SW is memory-efficient while maintaining performance for context lengths well beyond four times longer.

We also analyze the perplexity of tokens in various positional buckets within a long context, as depicted in Figure~\ref{fig:pg19}. In the language modeling task, SW demonstrates the ability to effectively manage contexts four times longer than standard models, although it is less competitive in relatively short contexts. While $\sparsek$+SW fails at contexts extending to 26k tokens, it outperforms both NTK and fine-tuned full attention models.

\subsection{Downstream Task}

We evaluated our method on the English subsets of LongBench~\cite{Bai2023LongBenchAB} using the OpenCompass package~\cite{2023opencompass}, which encompasses a wide range of long-context downstream tasks. The choice of language is based on the fact that the training corpus of TinyLlama is primarily in English. We test all models using greedy decoding, with the evaluation context size set to 8192.

All results are presented in Table~\ref{tab:longbench}. Full attention offers the best performance but incurs the highest memory cost. Sliding window attention is memory-efficient; however, it results in significant performance degradation. In contrast, our $\sparsek$+SW attention not only demonstrates strong performance but also achieves high memory efficiency. Notably, $\sparsek$+SW outperforms the training-free method, NTK, and the inference-time KV cache compression method, H2O~\cite{Zhang2023H2OHO}. This suggests the benefits of maintaining consistency between training and inference. However, $\sparsek$+SW underperforms fine-tuned full attention, representing a trade-off between efficiency and performance.

\begin{table}[ht]
    \caption{Results on LongBench. $^*$ denotes that the method is training-free. $^\dagger$ We use 512 globel (heavy-hitter) KV cache and 512 local KV cache in H2O.}
    \label{tab:longbench}
    \vskip 0.15in
    \centering
    \resizebox{0.99\columnwidth}{!}{%
        \setlength{\tabcolsep}{2pt}
        \begin{tabular}{c|ccccccccccccccccc} \toprule
            \multirow{2}{*}{\textbf{Model}} & \multicolumn{3}{c}{\textbf{Single-Doc QA}} & \multicolumn{3}{c}{\textbf{Multi-Doc QA}} & \multicolumn{3}{c}{\textbf{Summarziation}} & \multicolumn{3}{c}{\textbf{Few-shot Learning}} & \multicolumn{2}{c}{\textbf{Synthetic}} & \multicolumn{2}{c}{\textbf{Code}} & \multirow{2}{*}{\textbf{Avg.}}   \\ \cmidrule(l{2pt}r{2pt}){2-4}\cmidrule(l{2pt}r{2pt}){5-7}\cmidrule(l{2pt}r{2pt}){8-10}\cmidrule(l{2pt}r{2pt}){11-13}\cmidrule(l{2pt}r{2pt}){14-15}\cmidrule(l{2pt}r{2pt}){16-17}                                                                      
                           & NQA  & Qspr & MulFi & HQA & WMQA & Musq & GRpt & QMSM & MulN & TREC & TriQA & SMSM & PsgC & PsgR & LCC & Repo & \\\midrule
            \makecell{NTK$^*$ \\ $w=8192$} & 4.34 & 10.30 & 14.54 & 6.49 & 9.19 & 3.49 & 11.77 & 7.84 & 3.62 & 49.5 & 55.17 & 22.66 & 1.21 & 3.38 & 52.19 & 48.90  & 19.04  \\\midrule
            \makecell{Full \\ $w=8192$} & 3.95 & 13.07 & 13.16 & 6.81 & 10.77 & 3.51 & 15.17 & 6.12 & 8.30 & 61.00 & 65.15 & 26.02 & 0.39 & 2.37 & 56.72 & 50.36 & 21.42 \\\midrule
            \makecell{Full \\ H2O$^\dagger$} & 7.66 & 9.33 & 13.73 & 6.36 & 10.23 & 3.26 & 12.10 & 7.00 & 0.87 & 51.00 & 54.92 & 18.31 & 2.39 & 2.62 & 41.66 & 43.24 & 17.79 \\\midrule
            \makecell{SW \\ $w=1024$} & 1.34 & 8.69 & 5.41 & 2.76 & 4.46 & 0.48 & 11.78 & 4.25 & 2.39 & 25.50 & 13.43 & 5.33 & 2.3 & 0.50 & 52.22 & 27.50 & 10.52 \\\midrule
            \makecell{SparseK+SW\\$k=w=512$} & 5.19 & 14.29 & 13.24 & 6.85 & 9.21 & 3.83 & 14.11 & 5.97 & 5.85 & 55.00 & 52.06 & 24.79 & 0.61 & 2.61 & 53.90 & 50.89 & 19.90 \\\bottomrule
        \end{tabular}%
    }
\end{table}

\section{Conclusion}

We propose $\sparsek$ attention, a new approach to sparse attention that achieves both computational and memory efficiency. Within self-attention, we use an additional scoring network evaluating the importance of each key-value pair and select the top-$k$ pairs. We propose the differentiable $\sparsek$ operator, a relaxation of $\textsc{TopK}$, to enable gradient-based optimization. Experiments on language modeling and downstream tasks demonstrate consistent improvements compared to previous efficient attention methods.

\newpage

{\small
\bibliographystyle{plainnat}
\bibliography{custom,anthology}}


\appendix

\section{Derivation of SparseK}
\label{sec:proof}

\subsection{The solution of SparseK}

\begin{proposition}
    The solution of $\sparsek\coloneqq \argmin_{\vp\in \mathbb{C}}|| \vp-\vz ||^2$ is in the form of $\vp^* = \max(\min(\vz - \tau(\vz), \bm{1}), \bm{0})$, where $\mathbb{C} = \{\vp | \bm{0} \le \vp \le \1,  \1^\top\vp=k\}$. Define $F^*(\vz) = \{j| p^*_j = 1\}$ and $S^*(\vz) = \{j| 0 < p^*_j < 1\}$. Then,\begin{align}
        \tau(\vz) 
    &= \frac{\sum_{j\in S^*(\vz)}z_j + |F^*(\vz)| - k}{ |S^*(\vz)| }\label{eq:sparse_k_proof_part1}
    \end{align}
\end{proposition}

\begin{proof}
Consider the Lagrangian problem of $\sparsek$:
\begin{equation}
\mathcal{L}(\vz, \bm{\mu}, \bm{\nu}, \tau) = \frac{1}{2} ||\vp-\vz ||^2 - \bm{\mu}^\top\vp + \bm{\nu}^\top(\vp-\1) + \tau (\1^\top\vp - k)
\end{equation}

The optimal ($\vp^*, \bm{\mu}^*, \bm{\nu}^*, \tau^*$) must satisfy the following Karush-Kuhn-Tucker conditions:
\begin{align}
    \vp^* - \vz - \bm{\mu}^* + \bm{\nu}^* + \tau^*\1 = \bm{0}, \label{eq:kkt1}                                               \\
    \bm{0} \le \vp^* \le \1,\quad \1^\top\vp^* =k, \quad \bm{\mu}^* \ge \bm{0}, \quad \bm{\nu}^* \ge \bm{0}, \label{eq:kkt2} \\
    \mu^*_ip^*_i = 0,\quad v^*_i(p^*_i - 1) = 0,\quad \forall i \in [K]. \label{eq:kkt3}
\end{align}

\paragraph{Case I} If for $i \in [K]$ we have $0 < p^*_i < 1$, then from (\ref{eq:kkt3}) we must have $\mu^*_i=0$ and $\nu^*_i=0$, which from (\ref{eq:kkt1}) we must have $p^*_i =  z_i - \tau^*$, i.e., $z_i-1 < \tau^*$.

\paragraph{Case II} If for $i \in [K]$ we have $\mu^*_i > 0$, then from (\ref{eq:kkt3}) we must have $p^*_i=0$ and $\nu^*_i=0$, which from (\ref{eq:kkt1}) we must have $\mu^*_i = \tau^* - z_i > 0$, i.e., $z_i \le \tau^*$.

\paragraph{Case III} If for $i \in [K]$ we have $\nu^*_i > 0$, then from (\ref{eq:kkt3}) we must have $p^*_i=1$ and $\mu^*_i=0$, which from (\ref{eq:kkt1}) we must have $\nu^*_i = z_i - 1 - \tau^* > 0$, i.e., $z_i - 1\ge \tau^*$.

From \textbf{Case I, II, III} and (\ref{eq:kkt2}) we obtain
\begin{align}
    & \vp^* = \max(\min(\vz - \tau^*, 1), 0)\label{eq:sparse_k_sol},\\
     \text{and}\quad & \sum_{j\in S(\vz)} (z_j - \tau^*) + |F(\vz)| =k.\label{eq:lag_sol}
\end{align}

Rearrange (\ref{eq:lag_sol}) then we will get (\ref{eq:sparse_k_proof_part1}). 

\end{proof}

Let $z_{(1)} > z_{(2)} > \dots > z_{(m)}$ be the sorted coordinates of $\vz$. We can define $u^*=|F^*(\vz)|$ and $w^*=|S^*(\vz)| + |F^*(\vz)|$, so we have
\begin{align}
    z_{(u^*)} \ge \tau(\vz)+1 >  z_{(u^*+1)} &  & z_{(w^*)} > \tau(\vz) \ge z_{(w^*+1)} \label{eq:sol_constraint}
\end{align}
Consequently, (\ref{eq:sparse_k_proof_part1}) can be rewritten as
\begin{align}
    \tau(\vz)= \frac{\sum_{u^* < j \le w^*} z_{(j)} + u^* - k}{ w^* - u^* },\label{eq:run_sol_appendix}
\end{align}

\subsection{Get Algorithm~\ref{alg:line_eval_sol}}

The exact solution of (\ref{eq:run_sol_appendix}) can be evaluated by searching $(u,w)$ in the descending order that satisfies (\ref{eq:sol_constraint}). The descending order is given by the step functions with respect to an incremental threshold $ z_{(m)} - 1 < \beta \le z_{(1)}$,
\begin{align}
    u(\beta) = \max\{j | z_{(j)} \ge \beta + 1\} &  & w(\beta) = \max\{j | z_{(j)} > \beta\}.
\end{align}

Note that there are at most $2m$ distinct $(u, w)$ pairs, corresponding to the values of $\beta$ that may trigger change of either $u^*$ or $w^*$, i.e., $\{z_1,\dots, z_m\} \cup \{z_1 - 1,\dots, z_m-1\}$.

\subsection{Gradient}

The $\sparsek$ operator is differentiable everywhere except at splitting points causing changes to set $S^*(\vz)$ or $F^*(\vz)$. Note that we have $j\in S^*(\vz) \Leftrightarrow \tau(\vz) < z_j < \tau(\vz) + 1$. Then, from Equation (\ref{eq:sparsek_sol_main}) and Equation (\ref{eq:sparse_k_proof_part1}), we have
\begin{align}
    \frac{\partial \sparsek_i(\vz)}{\partial z_j} = \left\{\begin{aligned}
                                                                & \delta_{ij} - \frac{1}{|S^*(\vz)|} &  & \text{if } z_i \in S^*(\vz), \\
                                                                & 0                                  &  & \text{otherwise,}
                                                           \end{aligned}\right.
\end{align}
where $\delta_{ij}=1$ if $i=j$ and 0 otherwise. %
Then, the Jacobian matrix $\bJ(\vz)$ and Jacobian-vector product (JVP) for a given vector $\vv$ are given by
\begin{align}
    &\bJ(\vz) = \text{Diag}(\vs) - \vs\vs^\top / |S^*(\vz)| \\
    &\bJ(\vz)\cdot \vv = \vs\odot(\vv - \hat{v}\bm{1}), \text{with } \hat{v} = \frac{\sum_{j\in S^*(\vz)} v_j}{|S^*(\vz)|}\label{eq:jvp}
\end{align}
where $\vs$ is an indicator vector taking $1$ if $i \in S^*(\vz)$ and $0$ otherwise, and $\odot$ is the Hadamard product.
Note that $\sparsek$ has the same form of gradients, Jacobian and jacobian-vector product (JVP) as $\textsc{SparseMax}$~\cite{Martins2016FromST} but with a different definition of $S^*(\vz)$.

\section{Other Technical Details}

\subsection{Combine SparseK with Low-rank Linear Attention}
\label{sec:combine_with_linear_attn}

For simplicity, we consider the attention of one query, $\vq_i$. In \citet{Chen2021ScatterbrainUS}, each attendable position $j$ ($j \leq i$) is attended to via either sparse attention or linear attention. This hard strategy resembles top-$k$ attention. In our $\sparsek$ attention, we employ soft gating, denoted as $\vm^{sparsek}$. Consequently, the combination of $\sparsek$ and linear attention forms an interpolation between them, modulated by $\vm^{sparsek}$:
\begin{align}
    \vo_i = \sum_j \frac{\left(\vm_{sparsek,j}\exp\left(\vq_i^\top \vk_j\right) + \left(1-\vm_{sparsek,j}\right) \phi(\vq_i)^\top\phi(\vk_{j'})\right)\vv_j }{\sum_{j'}\vm_{sparsek,j}\exp(\vq_i^\top \vk_{j'}) + (1-\vm_{sparsek,j}) \phi(\vq_i)^\top\phi(\vk_{j'})}
\end{align}

We can use this notation to express other kinds of combination: (i) Linear+SW means to use a local gating $\vm_{local}$ where $\vm_{local,j} = 1$ if $i-j \le w$ and 0 otherwise. (ii) SparseK+Linear+SW means using a concatenation of $\vm_{local}$ and $\vm_{sparsek}$ as mentioned in Section~\ref{sec:combine_sparsek_and_local}. 

In this work, we adopt the feature map $\phi(\vx) = \text{elu}(\text{head-wise-linear}(\vx)) + 1$ based on the methodology proposed by \citet{Katharopoulos2020TransformersAR}. The additional head-wise linear transformation is introduced to mitigate the dilemma of $\bW_q$ and $\bW_k$ having to perform two types of attention. This adjustment introduces less than 1\% additional parameters. However, using linear attention can introduce significant overhead during training because $\text{head-wise-linear}(\vx)$ consumes substantial memory that needs to be stored in the computation graph. In contrast, the combination of $\sparsek$ and sliding window (SW) techniques introduces minimal overhead since the same $\bQ$, $\bK$, and $\bV$ are employed in both kinds of attention.

In our early experiments, we also considered the more recent Hedgehog feature map~\cite{Zhang2024TheH}. We observed slightly better performance compared to the aforementioned feature map. However, the Hedgehog feature map requires significantly more computational resources due to doubling the hidden dimension. While it is possible to fuse the feature map and linear attention within a single kernel, we leave this for future work.

Recently, Infini-attention~\cite{Munkhdalai2024LeaveNC} was proposed to integrate local attention and linear attention through an input-independent gating mechanism. We contend that our methods exhibit a superior expressive capability because our gating mechanism is data-dependent.

\paragraph{The numerical explosion issue}

Without normalization, we have observed that the gate $(1-\mathbf{m}_{sparsek,j})$ is trained to put all credits away from linear attention part $\phi(\mathbf{q}_i)^\top\phi(\mathbf{k}_j)$ (i.e., set $(1-\mathbf{m}_{sparsek,j})$ to zero) because standard dot-product attention is generally better than linear attention, and in contrast, the linear attention part is trained to increase $\phi(\mathbf{q}_i)^\top\phi(\mathbf{k}_j)$ significantly to survive. The score normalization breaks the competition and makes the training feasible.

\subsection{Efficient GPU Kernel of SparseK Attention}
\label{seq:eff_impl}

In this section, we discuss optimizations for obtaining practically efficient sparse attention. We discuss a few challenges here and present our solution briefly. We measure time and memory cost on one NVIDIA A6000 GPU when $n=8192, k=1024, h=4, d=64$, and $G=128$ to show the effectiveness of each optimization.

\paragraph{High memory cost} A naive implementation is to select contexts before doing attention. However, each selection produces $\hat{\bK}_i,\hat{\bV}_i\in\sR^{k\times d}$, resulting in a memory cost of $2nkd$ for all selections. The cost is significant considering a reasonable size of selected contexts, such as $k=100$.

\underline{\textit{Solution}} \quad We implement a fused Triton kernel to load needed key-value pairs on demand within FlashAttention~\cite{Dao2022FlashAttentionFA}. By doing so, the memory cost is reduced from 8192 MB to 4.125 MB.

\paragraph{Underuse of hardware} Our method requires $n$ matrix-vector multiplications (MVM) of $\vq_i\hat{\bK}_i$ instead of one matrix-matrix multiplication (MMM) of $\bQ\bK^\top$ in pure self-attention. This fails to capitalize on the hardware-optimized matrix-matrix multiplication in modern GPUs.

\underline{\textit{Solution}} \quad We group $G$ successive queries to do attention jointly. This is efficient because the majority of selections remain unchanged in successive steps. With grouping, $n$ MVM are transformed into $\frac{n}{G}$ MMM and hence hardware can be fully utilized. By doing so, the running time is reduced from 3.09 ms to 195 us, where about 2 ms is due to the faster matrix-matrix multiplications, and the reset is credited to less IO.

\paragraph{IO-inefficiency} Storing and using selection scores $\vm_{sparsek}$ has an IO complexity of $O(nk)$.
Besides,
As $\textsc{SparseMax}$, jacobian-vector product (JVP) computation of $\sparsek$ requires multiple read/write operations.

\underline{\textit{Solution}} \quad Our approach involves storing only the unnormalized scores $\vu\in\sR^{n}$ and the threshold $\bm{\tau}\in\sR^{n}$. Then, $\vm_{sparsek}$ is computed dynamically as needed, akin to the re-computation technique in FlashAttention, which trades slow IO with fast computation. Regarding gradients, we extend the backward kernel of FlashAttention to include IO-aware JVP computation. Furthermore, we introduce additional IO optimizations for grouped queries, achieved by reducing intermediate results in a group instead of using a computing-then-reducing pipeline. By doing so, the running time is reduced from 12.07 ms to 1.46ms and the memory cost of backward pass is reduced from 192.13 MB to 13.53 MB.

\subsection{Practical Implementation of Algorithm~\ref{alg:multi_step}}

Algorithm~\ref{alg:multi_step} scans the sequence of scores $\vu$ and incorporates heaps, making it unsuitable for GPUs. Therefore, we implement the algorithm on the CPU, where it performs efficiently, even with millions of input elements. However, this raises the concern that frequent data transfers between the CPU and GPU might degrade efficiency. To address this, we leverage the asynchronous execution property of GPU kernels to conceal the data transfer overhead behind other large GPU kernels.

\section{Additional Experimental Results}
\label{sec:additional_results}

\subsection{Language Modeling from Scratch}
\label{sec:more_results_from_scratch}

For all methods, we use the AdamW optimizer~\cite{Loshchilov2017FixingWD} with a learning rate of $6 \times 10^{-4}$ and a weight decay of 0.1. We apply gradient clipping with a threshold of 1.0. Our learning rate follows a cosine schedule, incorporating 100 warm-up steps from an initial learning rate of $1 \times 10^{-6}$. The batch size is set to 512K tokens. 

We utilize only those documents whose lengths are equal to or exceed the training context length, ensuring that no data packing is applied. This approach simplifies data preprocessing because the incremental selection of $\sparsek$ on two unrelated sentences is not meaningful. There are more sophisticated strategies available, such as the similarity-based batching~\cite{Zhong2024LoryFD}, which we plan to explore in future work. For evaluation, we compute perplexity on a held-out set using a sliding window approach without overlap. 

For GLA,  we adjusted the settings of its 340M architecture as follows: the hidden dimension was decreased from 1024 to 768, the number of layers was reduced from 24 to 12, and the number of heads was decreased from 4 to 2. For RetNet, the modifications were similar to those of the GLA, except that the number of heads was reduced to 3. For GLA and RetNet, we  use the same training hyperparameters (such as learning rate scheduler and weight decay) as in our experiments of Transformers rather than the official configuration~\cite{Yang2023GatedLA}. We found that our configuration leads to slightly better results. For Hash attention, we normalize $\bK$ as suggested in \citet{Pagliardini2023FasterCA}. Table~\ref{tab:hp_ppl_from_scratch} lists all used sparsity configurations, and figure~\ref{fig:more_reuslts_from_scratch} plots perplexity achieved by each model relative to the average (or maximum, if applicable) number of attended KV pairs per query.

We use the flash-linear-attention package~\cite{yang2024fla} for our linear attention related experiments, including GLA, RetNet, Linear+SW and SparseK+Linear+SW. We use the xformers~\cite{xFormers2022} package for our Fixed and BigBird experiments. We use the open-sourced code of \citet{Pagliardini2023FasterCA} for our Hash and Random experiments.

\begin{table}[H]
    \caption{Sparsity configurations used to produce Table~\ref{tab:ppl_from_scratch} and Figure~\ref{fig:more_reuslts_from_scratch}. \# denotes \textit{the number of}.}
    \label{tab:hp_ppl_from_scratch}
    \vskip 0.15in
    \centering
    \resizebox{0.98\columnwidth}{!}{%
    \begin{tabular}{cc|c|c|c} \toprule
      \multirow{2}{*}{\textbf{Model}} & \multirow{2}{*}{\textbf{Hyperparameter}}  & \multicolumn{3}{c}{\textbf{Training Context Length}}  \\
       & & 1024 & 4096 & 8192 \\\midrule
       SW and Full & window size & 256, 512 & 256, 512, 1024 & 256, 512, 1024 \\\midrule
       SparseK+SW& window size / $k$ & \multicolumn{3}{c}{128/128, 256/256, 512/512} \\\midrule
       Fixed & \# local/global & 12/4, 24/8 & 24/1, 24/2, 24/8 & 24/8 \\\midrule
       BigBird & \# global/random/local & 2/4/12, 4/4/24 & \multicolumn{2}{c}{2/4/12, 4/4/24, 8/8/24}  \\\midrule
       Hash & num clusters & 2, 4 & 8, 16 & 8, 16 \\\bottomrule
    \end{tabular}%
    }
\end{table}

\begin{figure}[H]
    \centering
    \includegraphics[width=\textwidth]{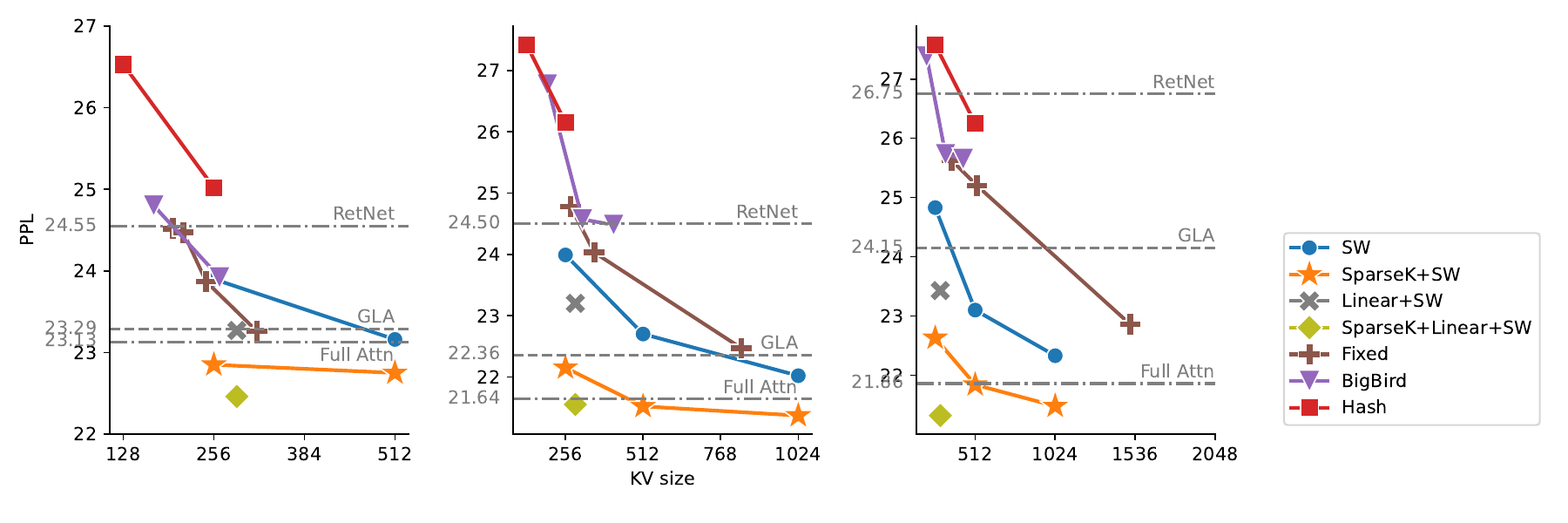}
    \caption{Training from scratch.}
    \label{fig:more_reuslts_from_scratch}
\end{figure}

\subsection{Fine-tuning Existing Models}
\label{sec:ft_hp}

For all methods, we use the AdamW optimizer with a weight decay of 0.1. We apply gradient clipping with a threshold of 1.0. Our learning rate follows a cosine schedule, incorporating 100 warm-up steps from an initial learning rate of $10\%$ of the peak learning rate, which is $3\times 10^{-4}$ for Pythia 160M, $1.5\times 10^{-4}$ for Pythia 410M and $1\times 10^{-4}$ for TinyLlama 1.1B. The batch size is set to 1M tokens.

For Pythia models, we use packing to generate training samples of the target length from a uniformly sampled subset of SlimPajama. For evaluation, we compute the perplexity on a held-out set using a sliding window approach with a step size of 512. For TinyLlama models, we utilize only those documents whose lengths are equal to or exceed the training context length in the up-sampled SlimPajama dataset~\citet{Fu2024DataEF}. For evaluation, we compute perplexity on a held-out set using a sliding window approach without overlap.

\subsection{More Downstream Task Evaluation on Fine-tuned Models}

We present the results of tasks that do not require long-context modeling in Table~\ref{tab:lm-eval}. These results were generated using the lm-eval-harness package~\cite{eval-harness}. Although $\sparsek$+SW still outperforms $SW$, a comparison with the original model indicates that fine-tuning on long data somewhat diminishes performance on short data~\cite{Fu2024DataEF}. This trade-off has been observed in other studies, prompting further research to address this issue.

\begin{table}[H]
    \caption{Results evaluated using lm-eval-harness.}
    \label{tab:lm-eval}
    \vskip 0.15in
    \centering
    \resizebox{0.98\columnwidth}{!}{%
        \begin{tabular}{l|cc|cccccccccccc} \toprule
            \textbf{Model}       & \textbf{Wiki.}   & \textbf{LMB.}    & \textbf{LMB.}  & \textbf{PIQA}  & \textbf{Hella.}      & \textbf{Wino.} & \textbf{ARC-e} & \textbf{ARC-c}       & \textbf{CoPA}  & \textbf{OBQA}        & \textbf{SciQA} & \textbf{BoolQA} \\
                                 & ppl $\downarrow$ & ppl $\downarrow$ & acc $\uparrow$ & acc $\uparrow$ & acc\_norm $\uparrow$ & acc $\uparrow$ & acc $\uparrow$ & acc\_norm $\uparrow$ & acc $\uparrow$ & acc\_norm $\uparrow$ & acc $\uparrow$ & acc $\uparrow$  \\\midrule
            Original             & 14.70            & 7.11             & 57.58          & 73.12          & 58.99                & 58.88          & 61.36          & 31.83                & 74.00          & 34.40                & 88.70          & 62.97           \\\midrule
            SW & 14.57 & 7.23 & 58.16 & 73.07 & 57.66 & 58.41 & 60.14 & 29.86 & 73.00 & 35.40 & 87.60 & 61.28 \\
            SparseK+SW & 14.02 & 7.40 & 58.33 & 72.58 & 57.52 & 59.27 & 59.01 & 29.35 & 77.00 & 34.60 & 87.90 & 61.83
            \\\bottomrule
        \end{tabular}%
    }
\end{table}

\subsection{Compare with Memory-efficient Attention Method}

We compare our method with \citet{Anagnostidis2023DynamicCP} in Table~\ref{tab:compare_with_DCP}, which uses a pairwise criterion to prune the KV cache based on the current query. Following \citet{Anagnostidis2023DynamicCP}, we use the GPT2 small architecture and train models with Wikipedia 20220301.en~\cite{wikidump} and BookCorpus~\cite{Zhu_2015_ICCV} datasets. DCP achieves a slightly better perplexity than ours, matching the stronger expressiveness of their method than our query-independent selection. However, their method costs a much longer training time due to the heavy pairwise criterion, making it hard to adopt in training on long documents. Meanwhile, our method remains efficient in this short context length experiments and achieves similar performance gains.

\begin{table}[H]
\caption{Perplexity evaluation on the Wikipedia 20220301.en and BookCorpus held-out set. DCP:~\cite{Anagnostidis2023DynamicCP}. SW: sliding window attention. The training context length is 1024. We set $\gamma=0.03$ to DCP, $w=256$ to SW, $k=w=128$ to our $\sparsek$+SW. This results in a similar sparsity among these models. We also report the training time relative to full attention.}
\label{tab:compare_with_DCP}
\vskip 0.15in
\centering
\begin{tabular}{c|ccc} \toprule
    & \textbf{DCP} & \textbf{SW} & \textbf{$\sparsek$+SW} \\\midrule
    \textbf{PPL} & 17.91 & 18.36 & 18.06 \\\midrule
    \textbf{Relative Training Time} & $264\%$ & $97\%$ & $98\%$ \\\bottomrule
\end{tabular}
\end{table}

\subsection{Ablation Study}

Previously, we choose $\sparsek$+SW as the default model. In this section, we conduct ablation study to investigate the performance of $\sparsek$-only model and effectiveness of learnable selection. We use the Pythia-160M architecture. The context length is 8192. All models are set to have a sparsity ratio similar to SW with a window size of 1024. We randomly initialize all parameters. All models are trained with 5000 steps.

Table~\ref{tab:ablation} shows the results. We first study attention types. SparseK-only outperforms sliding window attention. A combination of these two attention leads to better performance. Then, we study how to select. We consider two baselines: random selection and unlearnable selection which is parameterized as in Section~\ref{sec:faster_and_more_stable_training} \textit{initialization to mimic attention score}. The two baseline obviously underperforms SparseK-only, validating the effectiveness of our learnable selection.

\begin{table}[H]
    \centering
    \caption{Ablation study}
    \label{tab:ablation}
    \vskip 0.15in
    \begin{tabular}{c|ccccc} \toprule
       \textbf{Model}  & $\sparsek$+SW & $\sparsek$-only & \makecell{$\sparsek$-only \\random} & \makecell{$\sparsek$-only \\unlearnable} & SW \\\midrule
        \makecell{\textbf{PPL}\\w/ slope} & 27.38 &  27.91 &  33.72 & 36.20 & \multirow{2}{*}{28.12} \\\cmidrule{1-5}
        \makecell{\textbf{PPL}\\w/o slope} & 27.41 & 40.48 & 54.43 & 38.55 \\\bottomrule
    \end{tabular}

\end{table}

\subsection{Balance Between Selection Size and Sliding Window Size}

In Table~\ref{tab:analysis_k_w}, we present a comparison of various ratios of $k$ and $w$. Pythia 410M is utilized as the base model, with a training context length of 16,384. The results indicate an optimal performance at approximately $k=w=1024$. Consequently, we set $k=w$ in all subsequent experiments for consistency and simplicity.

\begin{table}[H]
    \caption{Comparison on different ratios of $k$ and $w$ in $\sparsek$+SW. }
    \label{tab:analysis_k_w}
    \vskip 0.15in
    \centering
    \begin{tabular}{c|cccccccc}\toprule
       $k$  & 0  & 128 & 256 & 512 & 768 & 1024 &1280 & 1536\\
       $w$ & 2048 & 1920 & 1792 & 1536 & 1280 & 1024 & 768 & 512\\\midrule
        \textbf{PPL} & 13.73 & 13.65 & 13.64 & 13.62 &  13.59 & 13.59 & 13.61 & 13.61 \\\bottomrule
    \end{tabular}
\end{table}

\subsection{Multi-group Selection}
\label{sec:multi_head_selection}

It is straightforward to use head-specific selection or grouped selection; the only modification required is to employ separate scoring networks for each head or group. In Table~\ref{tab:multi_group}, we compare the performance of multi-group selection against single-group selection. We use Pythia 410M as the base model. The training context length is 4096. $w$ is set to 256. When utilizing two groups with \(k=256\), the worst-case memory budget amounts to \(256 \times 2 = 512\). Our findings indicate that even with a 50\% increase in the budget, amounting to 768  (where $k=384$), the multi-group selection fails to achieve the performance of single-group attention. Consequently, we have opted to use single-group selection in our subsequent experiments.

\begin{table}[H]
    \caption{Study multi-group selection.}
    \label{tab:multi_group}
    \vskip 0.15in
    \centering
    \begin{tabular}{c|ccc}\toprule
        \textbf{Num of groups} & 1 & 2 & 2 \\
        $k$ & 512 & 256 & 384 \\ \midrule
        \textbf{PPL} & 13.63 & 13.74 & 13.69 \\\bottomrule
    \end{tabular}

\end{table}

\subsection{Speed Benchmark}

Figure~\ref{fig:benchmark1} presents an sole comparison of the wall-clock time between our developed Triton kernel (with $k=512$ and $w=512$) and the Triton implementation of FlashAttention-2~\cite{Dao2023FlashAttention2FA}, across different input lengths. While our kernel initially shows slower performance than FlashAttention-2 on short inputs due to overhead associated with selecting contexts, it outpaces FlashAttention-2 as input lengths surpass 4096 and 8192 for forward only and forward+backward, respectively. This performance gain is attributed to our method's linear complexity. Additionally, using a straight-through estimator can reduce the IO operations required for reading $\vm_{sparsek}$, thereby providing further improvements.

\begin{figure}[H]
    \centering
    \includegraphics[width=0.5\textwidth]{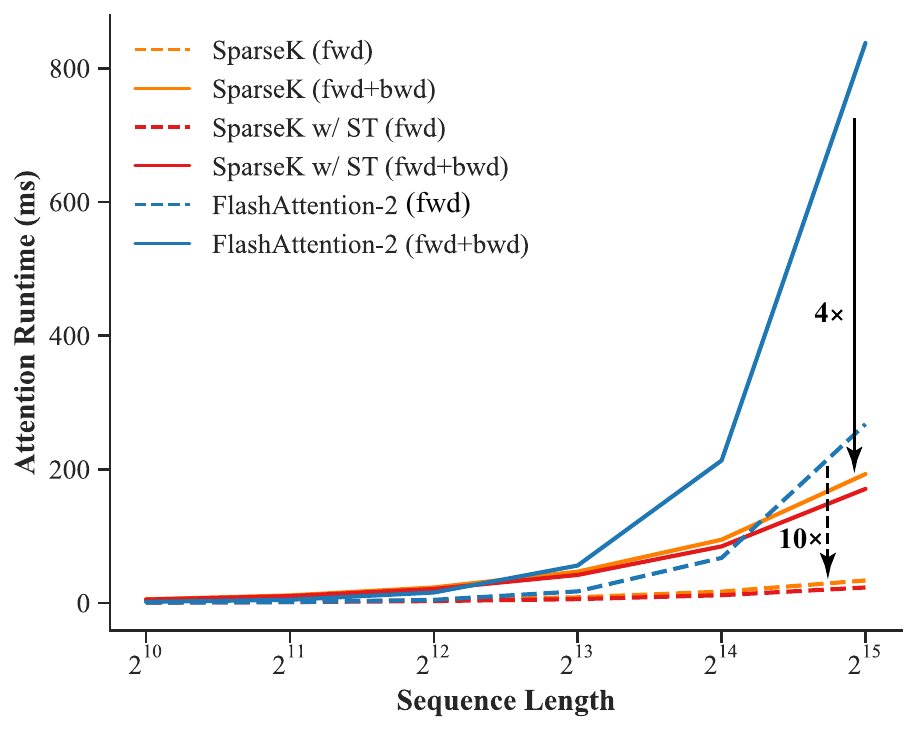}
    \caption{Benchmark the speed against FlashAttention-2. ST indicates the straight-through estimator}
    \label{fig:benchmark1}
\end{figure}

Figure~\ref{fig:time_to_reach_ppl} presents the times and steps required to achieve a specific perplexity. We use a pretrained Pythia-160m model with learning rate decay disabled. The results indicate that $\sparsek$+SW typically learns more quickly than the sliding window attention, achieving a better pareto-optimum in performance vs. efficiency.

\begin{figure}[H]
    \centering
    \includegraphics[width=\textwidth]{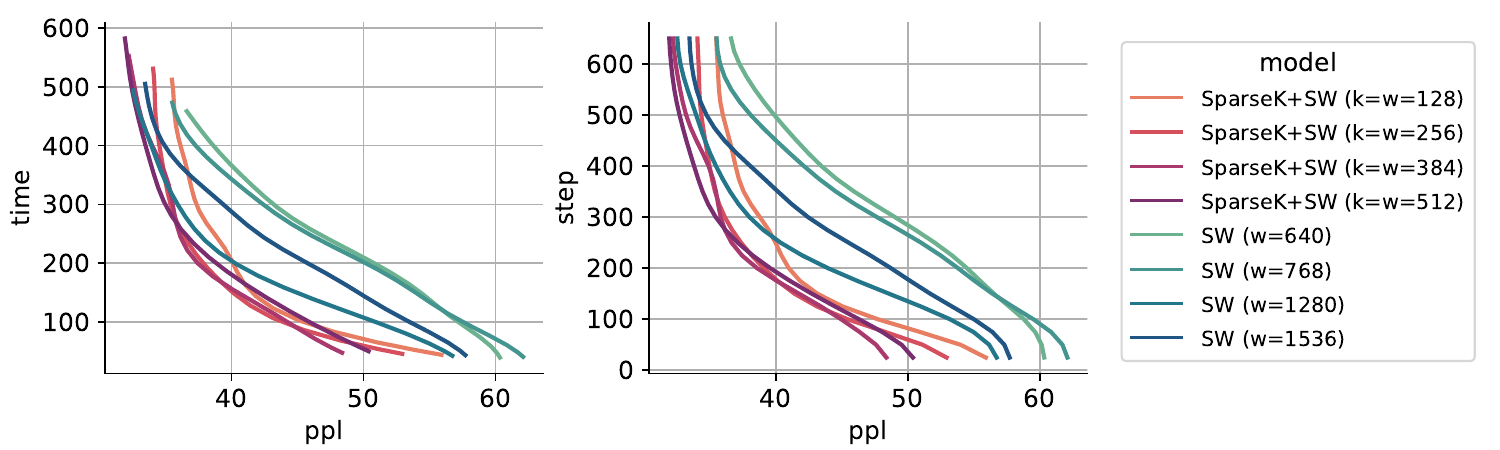}
    \caption{Time and steps used to reach ppls.}
    \label{fig:time_to_reach_ppl}
\end{figure}

In Table~\ref{tab:training_time}, we compare the training time when using different sparse attention. We use the Pythia-160M architecture.  The context length is 8192. All models are set to have a sparsity ratio similar to SW with a window size of 1024. All models are trained with 5000 steps.

\begin{table}[H]
    \caption{Relative training time (RTT) to complete a fixed-step training.}
    \label{tab:training_time}
    \vskip 0.15in
    \centering
    \begin{tabular}{c|ccccc}\toprule
       \textbf{Method}  & Full & SW & SparseK+SW & Random & Hash  \\\midrule
       \textbf{RTT}  & 100\% & 68.9\% & 83.2\% & 92.8\% & 113.2\% \\\bottomrule
    \end{tabular}
\end{table}

\subsection{Visualization}

\begin{figure}[H]
    \centering
    \includegraphics[width=\textwidth]{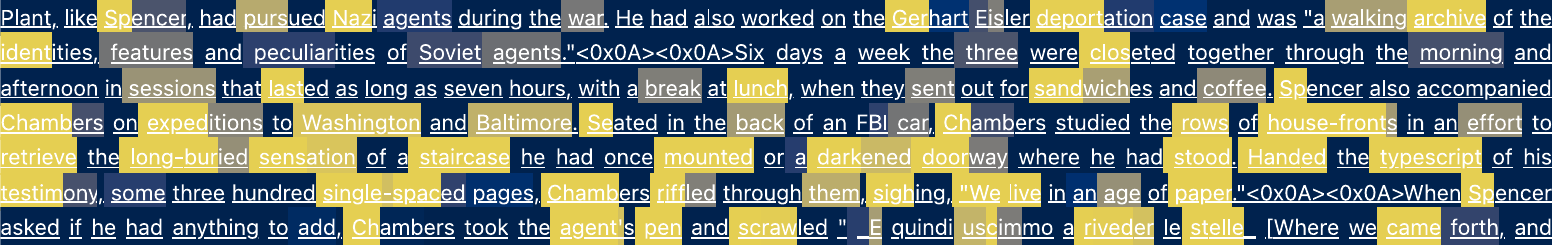}
    \caption{Visualization of the soft selection in an example sequence. Ligher background color means higher selection scores. We choose the 21st layer in a finetuned TinyLlama 1.1B $\sparsek$+SW model. To produce this visualization, we set $k=128, w=0$, which is different from the fine-tuning setting $k=512, w=512$, for a smaller figure size. }
    \label{fig:vis}
\end{figure}

\section{Limitations}
\label{sec:limitations}

This paper presents $\sparsek$ Attention, whose goal is to address both computational and memory efficiency challenges in long-range Transformer computing. We list limitations of our work below:
\begin{itemize}
    \item \textbf{Limited model size and context length} In this study, we validate the advantages of $\sparsek$ in settings involving models up to 1.1 billion parameters and context lengths up to 16,000 tokens. The primary reason for this limitation is our restricted computational resources. Despite these constraints, we consider a broad range of model sizes and context lengths to demonstrate the general applicability of our method. Nevertheless, with the advent of recent parameter-efficient fine-tuning techniques, it may be possible to experiment with our method on more limited devices. 
    \item \textbf{Decoder-only} We restrict our discussion to applying $\sparsek$ within Transformer decoders. the incremental evaluation capability of the $\sparsek$ operation demonstrates superior complexity compared to previous approaches. Nonetheless, the $\sparsek$ operation is also applicable in various other contexts, such as Transformer encoders and routing in mixture-of-expert models, where a top-k operation might be employed. The sparse output produced by the $\sparsek$ operation offers a notable advantage in these scenarios because it is close to the top-$k$ selection performed somewhere.
    \item \textbf{Text-only} We focus exclusively on text tasks. However, our method is not dependent on the input modality. Future research involving vision or speech could further substantiate the robustness of our method.
\end{itemize}

\section{Impact Statement}
\label{sec:Impact}

This paper presents $\sparsek$ Attention, whose goal is to address both computational and memory efficiency challenges in long-range Transformer computing. We believe our innovative attention mechanism can benefit both NLP and machine learning communities in constructing long-range foundation models. Specifically, we highlight the potential impacts of $\sparsek$ as follows:

\begin{itemize}
    \item \textbf{Efficient Long-Range Modeling.} First and foremost, the $\sparsek$ attention mechanism significantly reduces computational requirements compared to traditional self-attention mechanisms. By prioritizing a subset of key-value pairs,  $\sparsek$ attention effectively reduces the memory footprint without sacrificing model performance. This is particularly advantageous for resource-constrained environments and edge computing devices. Moreover, this innovation enhances training speed and convergence, contributing to more efficient model development and experimentation.

    \item \textbf{More powerful long-range pretrained models.} The differentiable $\sparsek$ operator facilitates gradient-based optimization, contributing to accelerated training speed. This is particularly advantageous for pretrained language models, where training on massive datasets can be time-consuming. Faster convergence and training speed enable researchers and practitioners to experiment with and refine models more efficiently. The efficiency gains provided by $\sparsek$ attention make it more scalable to work with massive long-sequence datasets, enabling researchers to harness the wealth of information available on the internet and other sources for training robust and contextually aware language models.

    \item \textbf{General Application to downstream tasks.} The proposed efficient self-attention mechanism can benefit a spectrum of NLP and machine learning downstream tasks, such as long-context document analysis and generation, Transformer-based long-form video analysis, \emph{etc.}
    Moreover, the $\sparsek$ attention mechanism is not limited to specific domains or tasks. Its adaptability makes it applicable across a wide range of natural language processing and machine learning applications, offering a versatile solution to improve efficiency and performance in diverse scenarios.
\end{itemize}

\end{document}